\documentclass[twoside]{article}

\usepackage[preprint]{aistats2026}
\usepackage{natbib}
\usepackage{mathtools, amsmath, amsfonts, amssymb,amsthm,multicol,derivative,enumitem,graphicx,float,tikz,tikz-qtree,forest,verbatim,bm,booktabs,lipsum}

\usepackage{algpseudocode} 
\usepackage[ruled,vlined]{algorithm2e}
\SetKwInput{KwInit}{Initialization}
\usetikzlibrary{fit,bayesnet, positioning, quotes,calc}
\usepackage[makeroom]{cancel}
\usepackage{multirow}
\usepackage{hyperref}
\theoremstyle{definition}

\theoremstyle{definition}

\theoremstyle{definition}

\theoremstyle{definition}

\theoremstyle{definition}

\theoremstyle{definition}

\theoremstyle{definition}

\theoremstyle{definition}

\definecolor{lightblue}{RGB}{173,216,230} 
\definecolor{lightgray}{gray}{.97}

\let\oldforall\forall
\renewcommand{\forall}{\oldforall \, }
\let\oldLeftrightarrow\Leftrightarrow
\renewcommand{\Leftrightarrow}{\oldLeftrightarrow \,\, }
\let\oldexist\exists
\renewcommand{\exists}{\oldexist \: }

\newcommand{\KL}{D_{\mathrm{KL}}}
\newcommand{\proj}{\mathcal{P}}

\newcommand{\E}{\mathbb{E}}

\newcommand{\cX}{\mathcal{X}}

\newcommand{\cD}{\mathcal{D}}

\newcommand{\R}{\mathbb{R}}

\newcommand{\Grad}{\nabla\!\!\!\!\nabla}

\newtheorem{assumption}{Assumption}
\newtheorem{proposition}{Proposition}

\newtheorem{theorem}{Theorem}
\newtheorem{remark}{Remark}

\usepackage[most]{tcolorbox}   % coloured / boxed content
\usepackage{xcolor}            % RGB / HTML colour definitions

\definecolor{bubblefill}{HTML}{CCE5FF} % light‑blue background
\definecolor{bubbleframe}{HTML}{6FA8FF} % (optional) darker frame

\newtcolorbox{contribbox}[2][]{%
	colback=bubblefill,
	colframe=bubbleframe,
	coltitle=black,
	boxrule=0pt,
	arc=2mm,
	left=6pt,right=6pt,top=6pt,bottom=6pt,
	fonttitle=\bfseries,
	title={#2},
	#1
}

\usepackage{xcolor}
\definecolor{darkblue}{rgb}{0,0,0.8}
\usepackage{pifont}     % in the preamble
\definecolor{lightblue}{RGB}{173,216,230} 
\definecolor{lightgray}{gray}{.97}

\begin{document}



\twocolumn[

\aistatstitle{Particle Dynamics for Latent-Variable Energy-Based Models}

\aistatsauthor{%
	Shiqin Tang$^{a}$\textsuperscript{*,†} \And
	Shuxin Zhuang$^{a,b}$\textsuperscript{*} \And
	Rong Feng$^{a,c}$ \And
	Runsheng Yu$^{d}$ \And
	Hongzong Li$^{e}$ \And
	Youzhi Zhang$^{a}$ 
}

\aistatsaddress{%
	$^{a}$Centre for Artificial Intelligence and Robotics, Chinese Academy of Sciences, Hong Kong\\
	$^{b}$Department of Data Science, City University of Hong Kong, Hong Kong\\
	$^{c}$Department of Computer Science, City University of Hong Kong, Hong Kong\\
	$^{d}$Department of Computer Science, Hong Kong University of Science and Technology, Hong Kong\\
	$^{e}$Generative AI Research and Development Center, The Hong Kong University of Science and Technology, Hong Kong\\[0.6em]
	\textsuperscript{*}\textbf{Equal contribution.} \quad
	\textsuperscript{†}\textbf{Corresponding author:} \texttt{shiqin.tang@cair-cas.org.hk}\\[0.3em]
	Emails: \texttt{shiqin.tang@cair-cas.org.hk}, \texttt{shuxin.zhuang@my.cityu.edu.hk}, \\ \texttt{rongfeng3-c@my.cityu.edu.hk}, \texttt{runshengyu@gmail.com}, \texttt{lihongzong@ust.hk},  \texttt{youzhi.zhang.lgc@gmail.com}%
}

 ]

\begin{abstract}
Latent-variable energy-based models (LV-EBMs) assign a single normalized energy to joint pairs of observed data and latent variables, offering expressive generative modeling while capturing hidden structure. We recast maximum-likelihood training as a saddle problem over distributions on the latent and joint manifolds and view the inner updates as coupled Wasserstein gradient flows. The resulting algorithm alternates overdamped Langevin updates for a joint negative pool and for conditional latent particles with stochastic parameter ascent, requiring no discriminator or auxiliary networks. We prove existence and convergence under standard smoothness and dissipativity assumptions, with decay rates in KL divergence and Wasserstein-2 distance. The saddle-point view further yields an ELBO strictly tighter than bounds obtained with restricted amortized posteriors. Our method is evaluated on numerical approximations of physical systems and performs competitively against comparable approaches.
\end{abstract}

\section{Introduction}

Energy-based models (EBMs) define a globally normalized probability density via an energy function, avoiding rigid decoder parameterizations while enabling flexible learning of complex structure \citep{Lecun2006EBM,DuMordatch2019,GrathwohlEtAl2019}. Throughout, we write $x$ for observed data and $z$ for latent states. In many domains, however, the observed signal $x$ is shaped by unobserved causes: style/content factors, nuisance variables, attributes, or even missing modalities. Latent-variable EBMs (LV-EBMs) address this mismatch by scoring pairs $(x,z)$ under a joint Gibbs density $p_\theta(x,z)\propto \exp\{-E_\theta(x,z)\}$, then marginalizing $z$ to obtain $p_\theta(x)$; the latent path grants both representational headroom and conditional control that EBMs on $x$ alone struggle to match. This paper develops a simple and principled learning framework for LV-EBMs grounded in optimal-transport dynamics, yielding an implementable particle algorithm with no discriminator, auxiliary posterior, or explicit decoder.

\paragraph{Why LV-EBMs over EBMs on $x$ alone?}
Firstly, explicit latents let the model factorize causes: by integrating over $z$, the marginal $p_\theta(x)=\int p_\theta(x,z)\,dz$ can represent richer multi-modality than a direct energy on $x$ \citep{yin2018semi, DuMordatch2019}. Secondly, LV-EBMs provide native conditionals $p_\theta(z\mid x)$ and $p_\theta(x\mid z)$ via clamping or conditional energies, enabling attribute manipulation, imputation, and semi-supervised learning without retrofitting \citep{GrathwohlEtAl2019}. In short, the latent states improve controllability while being integrated within a simple unnormalized energy function. 

\paragraph{Why LV-EBMs over common generative families?}
Likelihood-based decoders in VAEs are often paired with simple priors and amortized posteriors, which can limit expressivity or induce posterior collapse \citep{KingmaWelling2014,CremerEtAl2018}. Normalizing flows require invertibility and Jacobian constraints that trade off architectural freedom against tractable likelihoods \citep{RezendeMohamed2015,DinhEtAl2017}. Diffusion and score-based models excel at synthesis but rely on carefully scheduled noise processes and learned score fields over $x$ alone \citep{HoEtAl2020,SongErmon2019}. By contrast, LV-EBMs score joint configurations $(x,z)$ with a nonlinear energy, decoupling representation from decoder parameterization while retaining conditional pathways through $z$.

\paragraph{Our approach in a sentence.}
We cast maximum-likelihood learning for $p_\theta(x,z)$ as a saddle problem over distributions on the latent manifold and the joint $(x,z)$ manifold, and interpret the inner dynamics as coupled Wasserstein gradient flows \citep{JordanKinderlehrerOtto1998,AmbrosioGigliSavare2008,chewi2024statistical}. The flows induce explicit Fokker--Planck evolutions with practical overdamped Langevin updates \citep{Risken1996,WellingTeh2011}. Instantiated with particles, the algorithm alternates (i) per-example conditional latent moves approximating $p_\theta(z\mid x)$ and (ii) joint negative-sample refreshes for $p_\theta(x,z)$, followed by a stochastic ascent step on $\theta$. Under standard smoothness, dissipativity, and log-Sobolev-type conditions, these dynamics are well-posed and enjoy stability controls that tie KL/Wasserstein distances to the training objective \citep{Gross1975,BakryEmery1985,DurmusMoulines2017}.

\paragraph{Why particle dynamics over VI-based training of LV-EBMs?}
Classical latent-variable training builds and optimizes a separate approximate posterior (e.g., amortized Gaussian families) and often couples learning to a decoder network \citep{KingmaWelling2014}. Our formulation eliminates the auxiliary inference network and decoder: particles and Langevin steps realize the inner Wasserstein flows directly, using the energy to drive both joint and conditional updates. This avoids approximation bias from restrictive variational families while keeping the implementation compact. Conceptually, it parallels particle-based inference \citep{LiuWang2016} but operates in a coupled latent-and-joint space aligned with MLE for EBMs.

\paragraph{Paper roadmap.} 
Section~\ref{sec:related} reviews related work such as (LV)-EBMs and particle-based training. Section~\ref{sec:problem} sets up notation and assumptions. 
Section~\ref{sec:method} presents our training as a saddle problem realized by coupled Wasserstein gradient flows. 
Section~\ref{sec:theory} provides theoretical guarantees.
Section~\ref{sec:conditional} details practical conditional-generation procedures. Section~\ref{sec:experiments} reports empirical results on numerical physical simulations, with diagnostics for mixing and conditioning.

\begin{contribbox}{}
\textcolor{darkblue}{\textbf{Contributions.}} \textcircled{\scriptsize 1} We recast MLE as coupled Wasserstein gradient flows on the latent and joint $(x,z)$ manifolds, yielding explicit Fokker--Planck/Langevin inner updates. The algorithm is purely energy-driven—no discriminator, decoder, or auxiliary posterior.
\textcircled{\scriptsize 2} We provide several ways to implement conditional generations $p_\theta(x|z)$. 
\textcircled{\scriptsize 3} Under standard regularity, we prove well-posedness and convergence/stability of the coupled flows, and show our objective provides a tighter upper bound on NLL than variational-inference training. 
%\textcircled{\scriptsize 4} We include numerical studies on physical simulations to stress-test sampling, mixing, and conditional inference.
\end{contribbox}

\section{Backgrounds}
\label{sec:related}

\paragraph{Optimization on probability manifolds.}
In variational inference, one typically minimizes a Kullback--Leibler divergence over a parametric family,
\begin{align}
\min_{\phi \in \Phi}\;\KL\!\big(q_\phi(z)\,\|\,p^*(z)\big), \label{eq:vb_euc}
\end{align}
with $\phi \in \Phi \subset \R^m$ (see \citet{blei2017vi} for a review). This is a finite-dimensional Euclidean optimization problem. A complementary viewpoint—enabled by optimal transport—optimizes directly over probability measures endowed with the $L^2$–Wasserstein (Otto) Riemannian structure \citep{otto2001calculus,ambrosio2005gradient}:
\begin{align}
\min_{q \in \mathcal{P}_2(\R^\ell)}\;\KL\!\big(q(z)\,\|\,p^*(z)\big), \label{eq:vb_riem}
\end{align}
where $\mathcal{P}_2(\R^\ell)$ is the set of distributions with finite second moments. This measure-space formulation removes the restriction to a variational family and links the minimization of $\KL(\cdot\|p^*)$ to Wasserstein gradient flows: the steepest–descent dynamics for \eqref{eq:vb_riem} coincide with the Fokker--Planck/Langevin flow targeting $p^*$ and admit the JKO implicit–Euler scheme for time discretization \citep{jordan1998variational,ambrosio2005gradient,peyre2019computational}. In parallel, natural-gradient methods view Euclidean VI as Riemannian optimization on statistical manifolds equipped with the Fisher metric \citep{amari1998natural}, offering a complementary geometric perspective.

This manifold view naturally motivates particle procedures. \citet{kuntz2023particleEM} propose a particle algorithm to solve the classical Expectation Maximization (EM) problem with formulation
$\mathcal{P}_2(\R^\ell)$,
\begin{align} 
\max_{\theta\in\Theta, q \in \mathcal{P}_2(\R^\ell)}\, \E_{q(z)}[\log p_\theta(x|z)]-\KL\!\big(q(z)\,\|\,p_\theta(z|x)\big).\label{eq:particle_EM} 
\end{align}
and realize it via finite-particle gradient flows, alternating with an M-step that updates $\theta$ through stochastic gradient descent. Related extensions apply particle flows to enrich the variational family in semi-implicit VI (SIVI), improving posterior expressivity while retaining tractable updates (cf.\ the SIVI framework of \citet{yin2018semi} and particle-based refinements such as \citet{lim2024particle}).

\paragraph{Adversarial training and particle dynamics for EBMs.}
VERA \citep{grathwohl2021vera} frames EBM learning as a bi-level min--max problem: an entropy-regularized generator adversarially approximates the model’s negative phase, yielding a variational upper bound on NLL and stable training without Markov chains. In contrast to amortized variational approximations $q_\phi(x)$, \citet{neklyudov2021particle} replace the parametric family with distributions $q \in \mathcal{P}_2(\R^d)$ and construct deterministic particle flows whose finite-time transport tracks the evolving model density, effectively substituting short-run MCMC with Wasserstein gradient dynamics. %Building on these ideas, our method retains the energy-based formulation while (i) using finite-time particle/Langevin updates in place of a generator and (ii) coupling joint-state and latent moves for LV-EBMs, leading to a compact saddle-point algorithm aligned with maximum likelihood (see \S\ref{sec:method}).

\section{Problem Settings}
\label{sec:problem}

This section fixes notation and definitions and states the standing assumptions used throughout.

\paragraph{Ambient spaces and measures.}
Suppose latent variable $z$ and observed variable $x$ take values in $\R^\ell$ and $\R^d$, respectively. 
Let $\nu$ and $\mu$ denote the Lebesgue measures on $\R^d$ and $\R^\ell$, respectively; on the product we use $\nu\otimes\mu$.
Whenever gradients or Wasserstein geometry are invoked, we identify $\R^d\times\R^\ell\cong\R^{d+\ell}$ with its standard inner product $\langle\cdot,\cdot\rangle$ and norm $\|\cdot\|$.

\paragraph{Entropy functionals and Wasserstein metrics.}
For a probability measure $r$ absolutely continuous w.r.t.\ a reference measure $\lambda$, we write
\begin{align}
H_\lambda(r)\;:=\;-\!\int \log\!\Big(\frac{dr}{d\lambda}\Big)\,dr.\label{eq:entropy}
\end{align}
We denote by $W_{2,z}$ the $2$-Wasserstein metric on $\mathcal P_2(\R^\ell)$ and by $W_{2,j}$ the $2$-Wasserstein metric on $\mathcal P_2(\R^{d+\ell})$.
When the ambient space can be implied from context, we abbreviate $H_\mu(\cdot)$, $H_{\nu\otimes\mu}(\cdot)$ and write simply $W_2$.

\paragraph{Probability classes.} We use $\proj_2(\cX)$ to denote the set of Borel probability measures on $\cX$ with finite second moment and let 
\begin{align*}
&\mathcal Q := \big\{\, q\in\mathcal P_2(\R^\ell) :\ q\ll\mu,\ H_\mu(q)>-\infty \,\big\}\\
&\tilde{\mathcal Q} := \big\{\, \tilde q\in\mathcal P_2(\R^{d+\ell})  :\ \tilde q\ll(\nu\otimes\mu),\ H_{\nu\otimes\mu}(\tilde q)>-\infty \,\big\}
\end{align*}
For each data point $x_i$, we associate a measure $q_i\in\mathcal Q$ over $z$, while $\tilde q\in\tilde{\mathcal Q}$ is a joint pool over $(x,z)$.
Following contrastive EBM terminology (not supervised labels), we will refer to samples from $q_i$ as \emph{positive} and samples from $\tilde q$ as \emph{negative}; operational use of this terminology appears in Section~\ref{sec:method}.

\paragraph{Latent-variable EBM.}
For fixed $\theta$, let $E:\R^d\times\R^\ell\to\R\cup\{+\infty\}$ be measurable and define the Gibbs density
\begin{align}
p_\theta(x,z)\;=\;\frac{1}{Z(\theta)}\,\exp\!\big(-E(x,z;\theta)\big),
\label{eq:latent-ebm}
\end{align}
where the partition function $Z:\Theta \rightarrow \R_+$ is defined as 
\begin{align}
Z(\theta)\;=\;\int_{\R^d\times\R^\ell}\exp\!\big(-E(x,z;\theta)\big)\,(\nu\otimes\mu)(dx\,dz).\label{eq:partition}
\end{align}
Allowing $E = \infty$ encodes hard constraints. We denote the marginal distribution as  $p_\theta(x)=\int p_\theta(x,z)\,\mu(dz)$.

\begin{assumption}[Regularity and tails of the joint energy]
\label{assump:reg-joint}
$E(\cdot,\cdot;\theta)$ is $C^1$ with $L$-Lipschitz joint gradient and is $(m,b)$-dissipative:
\begin{align}
&\|\nabla E(u;\theta)-\nabla E(v;\theta)\|\le L\|u-v\|,\\
&\langle u,\nabla E(u;\theta)\rangle\ \ge\ m\|u\|^2-b,
\quad \forall u,v\in\R^{d+\ell}.
\end{align}
\end{assumption}

In particular, this assumption ensures that $Z(\theta)<\infty$ in \eqref{eq:latent-ebm}. Let $\cD = \{x^i: i\in [N]\}$ be the training set with empirical distribution $p_\cD(x) = \frac{1}{N}\sum_{i=1}^N \delta_{x^i}(x)$. 

\begin{assumption}[Functional inequalities ]
\label{assump:lsi}
For fixed $\theta$, the target $p_\theta$ satisfies a log-Sobolev inequality (LSI) with constant $\rho>0$ on $\R^{d+\ell}$.
Consequently, Talagrand’s $T_2(\rho)$ holds: $W_{2,j}^2(\tilde q,p_\theta)\le \tfrac{2}{\rho}\,\KL(\tilde q\|p_\theta)$ for all $\tilde q\in\tilde{\mathcal Q}$.
When $x$ is clamped, the conditional targets $p_\theta(\cdot\mid x)$ satisfy an LSI with constant $\rho_z>0$ on $\R^\ell$ (for $p_D$-a.e.\ $x$), yielding $W_{2,z}$–Talagrand bounds on $\mathcal Q$.
\end{assumption}

\begin{remark} In practice, we can parameterize $E$ as 
\begin{align}
E_\theta(x,z) = U_\theta(x,z) + \frac{\lambda_x}{2} \|x\|^2 + \frac{\lambda_z}{2} \|z\|^2,
\end{align}
with $\lambda_x, \lambda_z >0$ and keep $\nabla E_\theta$ L-Lipschitz via spectral norm or gradient clipping. 
This ensures normalizability and tail dissipativity; the quadratic envelope makes standard LSI and $T_2$ conditions plausible.
\end{remark}

\paragraph{Notational economy.}
We maintain two distinct Wasserstein spaces throughout:
the latent manifold $(\mathcal Q,W_{2,z})$ on $\R^\ell$ and the joint manifold $(\tilde{\mathcal Q},W_{2,j})$ on $\R^{d+\ell}$.
To avoid redundancy, we (i) suppress subscripts on $H$ and $W_2$ when unambiguous from the argument, and
(ii) use the same font $\nabla$ for gradients in $z$ or $(x,z)$, specifying the variable by a subscript only when needed (e.g., $\nabla_z E$, $\nabla_{x,z}E$).

\section{Particle Dynamics}
\label{sec:method}
In this section, we cast maximum-likelihood learning of the joint Gibbs model $p_\theta(x,z)$ as a saddle problem over distributions on the latent manifold and the joint $(x,z)$ manifold. 
This viewpoint replaces amortized inference with coupled Wasserstein gradient flows: per-example latent particles approximate $p_\theta(z|x)$ while a joint negative pool tracks $p_\theta(x,z)$. Discretizing the flows with overdamped Langevin dynamics yields a simple, fully energy-driven particle algorithm that alternates inner particle moves with a stochastic ascent step on $\theta$. We begin by deriving the saddle objective and its inner optima, then instantiate the corresponding flows and their particle updates.

\paragraph{Saddle-point MLE for LV-EBMs (``adversarial'' view).} 
We train the LV-EBMs by maximizing the log likelihood $p_\theta(x)$ represented in the following form,
\begin{align*}
\log p_\theta(x) &= \log \int \exp(-E(x,z;\theta)) dz \\
&\qquad - \log \iint \exp(-E(x,z;\theta)) dx dz\\
&= \log \E_{q(z)}\Big[\frac{\exp(-E(x,z;\theta))}{q(z)}\Big] \\
&\qquad - \log \E_{\tilde q(x,z)}\Big[\frac{\exp(-E(x,z;\theta))}{\tilde q(x,z)}\Big]\\
&= \sup_{q\in Q} \E_{q(z)}\Big[\log \frac{\exp(-E(x,z;\theta))}{q(z)}\Big] \\
&\qquad - \sup_{\tilde q \in \tilde Q}\E_{\tilde q(x,z)}\Big[\log \frac{\exp(-E(x,z;\theta))}{\tilde q(x,z)}\Big]\\
&= \sup_{q\in Q}\inf_{\tilde q \in \tilde Q}\Big\{ -\E_{q(z)}[E(x,z;\theta)] \\
&\quad + \E_{\tilde q(x,z)}[E(x,z;\theta)] + H(q) - H(\tilde q)\Big\},
\end{align*}
where $H(\cdot)$ represents entropy defined in \eqref{eq:entropy}.
We use the notation $q^{\neg i} = \{q_j: (j \in [N])\wedge (j \neq i) \}$ and $\bar q = \{q_j: j\in [N]\}$. 
Consequently, we establish the following saddle-point formulation:
\begin{equation}
\begin{aligned}
\max_{\theta, q^i\in Q} \min_{\tilde q \in \tilde Q} \quad \Big\{ &F(\bar q, \tilde q, \theta) := \E_{\tilde q(x,z)}[E(x,z;\theta)] - H(\tilde q)\\
& - \frac{1}{N}\sum_{i=1}^N \big(\E_{q^i(z)}[E(x^i,z;\theta)] - H(q^i) \big)\Big\}
\end{aligned}
\label{eq:main_eq}
\end{equation}
\normalsize
The first variations of $F$ are given by
\begin{equation}
\begin{aligned}
\frac{\delta F}{\delta q^i}(z) &= -\,E_\theta(x^i,z)\;-\;\log q^i(z)\;+\;c_i,\\
\frac{\delta F}{\delta \tilde q}(x,z) &= \;\;E_\theta(x,z)\;+\;\log \tilde q(x,z)\;+\;c,
\end{aligned}
\end{equation}
with $c_i$ and $c$ being constants. 
Following Otto calculus \citep{AGS08,Santambrogio15,JordanKinderlehrerOtto98}, we denote the Wasserstein gradient by $\Grad_{q} F \;\;:=\;\; \nabla\,\frac{\delta F}{\delta q}$.
Since each $q^i$ is maximized and $\tilde q$ is minimized in \eqref{eq:main_eq}, we take ascent for $q^i$ and descent for $\tilde q$:
\begin{align}
\Grad_{q^i} F(\bar q,\tilde q,\theta)
&= -\,\nabla_z E_\theta(x^i,z)\;-\;\nabla_z \log q^i(z),\label{eq:wgrad_qi_grad}\\
\Grad_{\tilde q} F(\bar q,\tilde q,\theta)
&= \;\;\nabla_{x,z} E_\theta(x,z)\;+\;\nabla_{x,z}\log \tilde q(x,z).\label{eq:wgrad_tq_grad}
\end{align}
Coupled with the continuity equation $\partial_t q_t=-\nabla\!\cdot(q_t v_t)$ and the choices
$v_t^i=\Grad_{q^i_t}F$ and $\tilde v_t=-\Grad_{\tilde q_t}F$, we obtain the Fokker--Planck equations (FPEs)
\small
\begin{align}
\partial_t q_t^i(z)
&= \nabla_z\!\cdot\!\big(q_t^i(z)\,\nabla_z E_\theta(x^i,z)\big) + \Delta_z q_t^i(z), \label{eq:fpe_qi}\\
\partial_t \tilde q_t(x,z)
&= \nabla_{x,z}\!\cdot\!\big(\tilde q_t(x,z)\,\nabla_{x,z} E_\theta(x,z)\big)
   + \Delta_{x,z}\tilde q_t(x,z),\label{eq:fpe_joint}
\end{align}
\normalsize
for $i \in [N]$.
Under standard smoothness and confining growth of $E_\theta$, \eqref{eq:fpe_qi}–\eqref{eq:fpe_joint} are the $W_2$–gradient flows of the corresponding free energies \citep{AGS08,JordanKinderlehrerOtto98}. 
The FPEs \eqref{eq:fpe_qi}–\eqref{eq:fpe_joint} are realized by the It\^o SDEs
\begin{align}
\mathrm d z_t^i &= -\,\nabla_z E_\theta(x^i,z_t^i)\,\mathrm dt + \sqrt{2}\,\mathrm d B_t^{(i)}, \label{eq:sde_qi}\\
\mathrm d \tilde z_t &= -\,\nabla_z E_\theta(\tilde x_t,\tilde z_t)\,\mathrm dt + \sqrt{2}\,\mathrm d B_t^{(z)}, \label{eq:sde_joint_z}\\
\mathrm d \tilde x_t &= -\,\nabla_x E_\theta(\tilde x_t,\tilde z_t)\,\mathrm dt + \sqrt{2}\,\mathrm d B_t^{(x)}, \label{eq:sde_joint_x}
\end{align}
where $B_t$ represents standard Brownian motions. 
Their marginal laws solve \eqref{eq:fpe_qi}–\eqref{eq:fpe_joint} in the weak sense \citep[Ch.~5]{Risken96,Pavliotis14}.

\begin{algorithm}[!t]
\caption{Coupled Langevin Particle Dynamics for LV-EBMs}
\label{alg:particle_dynamics}
\DontPrintSemicolon
\KwIn{data $\{x^i\}_{i=1}^N$; particle counts $M$ (latent), $D$ (joint); step sizes $\eta_z,\eta_{xz}$; learning rate $\alpha$; iterations $T$}
\KwInit{$z^{ik}_0 \sim \mathcal N(0,I_l)$ for all $i,k$;\\ $\tilde z^j_0 \sim \mathcal N(0,I_l)$,\quad  $\tilde x^j_0 \sim \mathcal N(0,I_d)$ for all $j$;}
\For{$t=0$ \KwTo $T-1$}{
  \tcp{Latent (conditional) Langevin: }%$q^i_{t+1} \approx \frac{1}{M}\sum_k \delta_{z^{ik}_{t+1}}$
  \For{$i=1$ \KwTo $N$}{
    \For{$k=1$ \KwTo $M$}{
      sample $\xi^{ik}_t \sim \mathcal N(0,I_l)$\;
      $\left.z^{ik}_{t+1} \leftarrow z^{ik}_{t} - \eta_z\,\nabla_z E(x^i, z^{ik}_t;\theta_t) + \sqrt{2\eta_z}\,\xi^{ik}_t\right.$
    }
  }
  \tcp{Joint (negative) Langevin: }%$\tilde q_{t+1} \approx \frac{1}{D}\sum_j \delta_{(\tilde x^j_{t+1},\tilde z^j_{t+1})}$
  \For{$j=1$ \KwTo $D$}{
    sample $\xi^{j,(x)}_t \sim \mathcal N(0,I_d)$,\, $\xi^{j,(z)}_t \sim \mathcal N(0,I_l)$;\; %\tcp*{equivalently one draw in $\R^{d+l}$}
    $\left.\tilde x^{j}_{t+1} \gets \tilde x^{j}_{t} - \eta_{xz}\,\nabla_x E(\tilde x^j_t,\tilde z^j_t;\theta_t) + \sqrt{2\eta_{xz}}\,\xi^{j,(x)}_t\right.$\;
    $\left.\tilde z^{j}_{t+1} \gets \tilde z^{j}_{t} - \eta_{xz}\,\nabla_z E(\tilde x^j_t,\tilde z^j_t;\theta_t) + \sqrt{2\eta_{xz}}\,\xi^{j,(z)}_t\right.$
  }
  \tcp{$\text{Parameter ascent on the energy contrast}$}
  $\displaystyle \theta_{t+1} \gets \theta_t
  + \alpha\,\Big[
  \frac{1}{D}\sum_{j=1}^D \nabla_\theta E(\tilde x^j_{t+1},\tilde z^j_{t+1};\theta_t)
  - \frac{1}{NM}\sum_{i=1}^N\sum_{k=1}^M \nabla_\theta E(x^i, z^{ik}_{t+1};\theta_t)
  \Big]$\;
  %\tcp{equivalently: $\theta_{t+1}=\theta_t - \alpha\,\nabla_\theta\!\left[\frac{1}{N}\sum_i \E_{q^i_{t+1}}E(x^i,z;\theta_t)-\E_{\tilde q_{t+1}}E(x,z;\theta_t)\right]$}
  \tcp{(Optional) persistence refresh: with prob.\ $\rho$, reinit a fraction of $\{(\tilde x^j,\tilde z^j)\}$ from $\mathcal N(0,I_d)\times\mathcal N(0,I_l)$}
}
\end{algorithm}

\paragraph{Particle algorithm.}
We approximate the Wasserstein flows using persistent particle systems: for each datum $x^i$ we maintain latent particles $\{z^{ik}_t\}_{k=1}^M$ evolving under conditional Langevin dynamics, and we maintain a joint negative pool $\{(\tilde x^j_t,\tilde z^j_t)\}_{j=1}^D$ evolving under joint Langevin dynamics. Particles are initialized i.i.d.\ but thereafter are correlated across iterations due to persistence. The corresponding empirical measures are
\[
\hat q^{\,i}_t \;=\; \frac{1}{M}\sum_{k=1}^M \delta_{z^{ik}_t}
\quad\text{and}\quad
\hat{\tilde q}_t \;=\; \frac{1}{D}\sum_{j=1}^D \delta_{(\tilde x^j_t,\tilde z^j_t)}\,,
\]
which approximate $q^i$ and $\tilde q$, respectively. At each iteration we take overdamped Langevin steps for all particles and perform a stochastic ascent step on $\theta$ that increases the energy contrast (negative minus positive):
\footnotesize
\begin{equation}
\begin{aligned}
\theta_{t+1}&=\theta_t - \alpha \nabla_\theta\Big[\frac{1}{N}\sum_i \E_{\hat q^{\,i}_{t+1}}E(x^i,z;\theta_t)%\\&\qquad\qquad\qquad\qquad    
-\E_{\hat{\tilde q}_{t+1}}E(x,z;\theta_t)\Big]
\end{aligned}
\end{equation}
\normalsize
The complete particle dynamics and parameter updates are summarized in Algorithm~\ref{alg:particle_dynamics}.

\section{Theoretical Guarantees}
\label{sec:theory}
%This section formalizes the guarantees behind the inner updates introduced in Section \ref{sec:method}. 
We first identify the inner optima of the saddle objective as joint and conditional Gibbs measures, then show that the associated Fokker–Planck (FP) flows contract exponentially in $\KL$ and $W_2$ under log–Sobolev inequalities (LSIs), formalizing that our Langevin inner loops are Wasserstein gradient flows toward the model targets \citep{jordan1998variational,AmbrosioGigliSavare2008}. We next describe a variational-inference (VI) training variant for LV-EBMs—obtained by restricting the inner distributions to amortized parametric families—and prove a tighter-than-VI result.

%We first identify the inner optima of the saddle objective as joint and conditional Gibbs measures, then show that the associated Fokker–Planck (FP) flows contract exponentially in $\KL$ and $W_2$ under log–Sobolev inequalities (LSIs), formalizing that our Langevin inner loops are Wasserstein gradient flows toward the model targets \citep{jordan1998variational,AmbrosioGigliSavare2008}. We next describe a variational-inference (VI) training variant for LV-EBMs—obtained by restricting the inner distributions to amortized parametric families—and prove a \emph{tighter-than-VI} result: the unrestricted (particle) saddle objective recovers the exact dataset log-likelihood, while any parametric VI restriction yields a weakly smaller ELBO, with strict gap under misspecification of either the conditional or joint family.

%This section formalizes the guarantees behind the inner updates introduced in Section \ref{sec:method}. We first identify the inner optima of the saddle objective as joint and conditional Gibbs measures, then show that the corresponding Fokker–Planck (FP) flows contract exponentially in $\KL$ and $W_2$ under log–Sobolev inequalities (LSIs). These statements make precise the intuition that our Langevin inner loops are Wasserstein gradient flows that drive the positive/negative distributions toward the model targets \citep{jordan1998variational,AmbrosioGigliSavare2008}. 
% <- update this later. 

%%%%%%%%%%%%% Inner optimum

\begin{proposition}[Inner optima for \eqref{eq:main_eq}: joint and conditional Gibbs principles]
\label{prop:inner-opt}
Under Assumption~1, for any $\tilde q\in\tilde{\mathcal Q}$ and any $x\in\R^d$, $q \in\mathcal Q$,
\begin{align}
&\E_{\tilde q}[E(x,z;\theta)] - H_{\nu\otimes\mu}(\tilde q) = \KL(\tilde q\|p_\theta) - \log Z(\theta),\\
&\E_{q}[E(x,z;\theta)] - H_\mu(q) = \KL(q \|p_\theta(\cdot\mid x)) - \log p_\theta(x).
\end{align}
Hence the inner optima in \eqref{eq:main_eq} are
\begin{align}
\tilde q^\star(x,z)=p_\theta(x,z) \quad\text{and}\quad q^\star(z\mid x)=p_\theta(z\mid x), 
\end{align}
and in each case the minimizer is unique up to $\nu\otimes\mu$ (resp.\ $\mu$) null sets.
\end{proposition}

\begin{proof}
Define the Radon–Nikodym derivative $\tilde r=\tfrac{d\tilde q}{d(\nu\otimes\mu)}$. We have
\begin{align*}
\KL(\tilde q\|p_\theta)&=\!\int \log\!\Big(\tfrac{\tilde r}{d p_\theta/d(\nu\otimes\mu)}\Big)\,d\tilde q\\
&=\!\int \log\tilde r\,d\tilde q + \!\int E\,d\tilde q + \log Z(\theta)\\
&= -H_{\nu\otimes\mu}(\tilde q) + \E_{\tilde q}[E] + \log Z(\theta).
\end{align*}
The conditional identity is identical with the product base measure replaced by $\mu$ and $p_\theta(\cdot\mid x)\propto e^{-E(x,\cdot)}$; $\log p_\theta(x)$ is the (negative) log-partition for the conditional. Uniqueness follows from $\KL(\cdot\|\cdot)\ge 0$ with equality iff the arguments agree a.s.
\end{proof}

\begin{theorem}[Convergence of inner Fokker–Planck flows]
\label{thm:inner-flow}
For a fixed $\theta$, assume:
\begin{itemize}
\item[(J)] (Joint LSI) The target $p_\theta$ on $\R^{d+\ell}$ satisfies a log–Sobolev inequality with constant $\rho>0$.
\item[(C)] (Conditional LSI) For every $x\in\R^d$ for which $p_\theta(\cdot\mid x)$ is defined, the conditional target on $\R^\ell$ satisfies a log–Sobolev inequality with constant $\rho_z>0$.
\end{itemize}
Let $\tilde q_t\in\tilde{\mathcal Q}$ solve the joint Fokker–Planck equation
\[
\partial_t \tilde q_t
= \nabla_{x,z}\!\cdot\!\big(\tilde q_t\,\nabla_{x,z} E(\cdot;\theta)\big) + \Delta_{x,z}\tilde q_t,
\quad \tilde q_{t=0}=\tilde q_0\in\tilde{\mathcal Q},
\]
and, for any fixed $x\in\R^d$, let $q_t^x\in\mathcal Q$ solve the conditional Fokker–Planck equation
\[
\partial_t q_t^x
= \nabla_{z}\!\cdot\!\big(q_t^x\,\nabla_{z} E(x,\cdot;\theta)\big) + \Delta_{z} q_t^x,
\quad q_{t=0}^x=q_0^x\in\mathcal Q.
\]
Then, for all $t\ge 0$,
\begin{align*}
&\KL(\tilde q_t\,\|\,p_\theta)
\le e^{-2\rho t}\,\KL(\tilde q_0\,\|\,p_\theta),\\
&W_{2,j}(\tilde q_t, p_\theta)
\le \sqrt{\tfrac{2}{\rho}}\;e^{-\rho t}\,\KL(\tilde q_0\,\|\,p_\theta)^{1/2},
\end{align*}
and, for every such $x$,
\begin{align}
&\KL\!\big(q_t^x\,\|\,p_\theta(\cdot\mid x)\big) \le e^{-2\rho_z t}\,\KL\!\big(q_0^x\,\|\,p_\theta(\cdot\mid x)\big),\\
&W_{2,z}\!\big(q_t^x, p_\theta(\cdot\mid x)\big)
\le \sqrt{\tfrac{2}{\rho_z}}\;e^{-\rho_z t}\,\KL\!\big(q_0^x\,\|\,p_\theta(\cdot\mid x)\big)^{1/2}.
\end{align}
\end{theorem}

\begin{proof}[Proof sketch]
For the joint flow, the entropy dissipation identity gives
$\frac{d}{dt}\KL(\tilde q_t\|p_\theta)=-\mathcal I(\tilde q_t\|p_\theta)$,
where $\mathcal I$ is the relative Fisher information.
By (J), $\mathcal I\ge 2\rho\,\KL$, hence
$\frac{d}{dt}\KL\le -2\rho\,\KL$; Grönwall yields the KL decay, and Talagrand’s $T_2(\rho)$ gives the $W_{2,j}$ bound \citep[Ch.~20]{Santambrogio15}.
The conditional case is identical with $x$ fixed and (C) replacing (J).
\end{proof}

\paragraph{Variational training for LV-EBMs.}
Variational inference (VI) for LV-EBMs restricts the variational families to parametric classes:
$q^i(z)\in \mathcal S_{\mathrm{pos}}(x^i)=\{q_\phi(z\mid x^i):\phi\in\Phi\}$ and
$\tilde q(x,z)\in \mathcal S_{\mathrm{neg}}=\{q_\psi(x,z):\psi\in\Psi\}$, e.g., Gaussian/amortized posteriors
and a tractable joint family for the global normalizer. The resulting (dataset) ELBO is then
\begin{equation}
\begin{aligned}
&\mathrm{ELBO}_{\mathrm{VI}}(\theta,\phi,\psi) := \big[\E_{q_\psi} E(x,z;\theta)-H(q_\psi)\big]\\
& + \frac{1}{N}\sum_{i=1}^N\big[-\E_{q_\phi(z\mid x^i)} E(x^i,z;\theta)+H(q_\phi(\cdot\mid x^i))\big],
\end{aligned}
\end{equation}
which is maximized jointly over $(\theta,\phi,\psi)$.
In contrast, particle dynamics allows $q^i$ and $\tilde q$ to range over nonparametric particle measures that, in the limit of enough particles/steps, approximate any element of $\mathcal Q$ and $\tilde{\mathcal Q}$.

\begin{theorem}
\label{thm:tighter}
Define, for each $x^i$,
\begin{align*}
&\mathcal A_i(q):= -\E_{q(z)}[E(x^i,z;\theta)] + H(q),\\
&\mathcal B(\tilde q):= -\E_{\tilde q(x,z)}[E(x,z;\theta)] + H(\tilde q).
\end{align*}
For any families $\mathcal S_{\mathrm{pos}}(x)\subseteq\mathcal Q$ and $\mathcal S_{\mathrm{neg}}\subseteq\tilde{\mathcal Q}$, let
\small
\[
\mathrm{ELBO}(\theta;\mathcal S_{\mathrm{pos}},\mathcal S_{\mathrm{neg}})
:= \frac{1}{N}\sum_{i=1}^N \sup_{q\in\mathcal S_{\mathrm{pos}}(x^i)} \mathcal A_i(q)
\;-\; \sup_{\tilde q\in\mathcal S_{\mathrm{neg}}} \mathcal B(\tilde q).
\]
\normalsize
Consider the following cases:
\begin{enumerate}[label=(\arabic*)]
\item \textbf{Exactness with unrestricted sets.}
If $\mathcal S_{\mathrm{pos}}(x)=\mathcal Q$ and $\mathcal S_{\mathrm{neg}}=\tilde{\mathcal Q}$, then
\begin{align*}
&\mathrm{ELBO}(\theta;\mathcal Q,\tilde{\mathcal Q}) \;=\; \frac{1}{N}\sum_{i=1}^N \log\!\int e^{-E(x^i,z;\theta)}dz\\
&\;-\; \log\!\iint e^{-E(x,z;\theta)}dx\,dz
\;=\; \frac{1}{N}\sum_{i=1}^N \log p_\theta(x^i).
\end{align*}
\item \textbf{Monotonicity \& domination.}
For any $\mathcal S_{\mathrm{pos}}(x)\subseteq\mathcal Q$ and
$\mathcal S_{\mathrm{neg}}\subseteq\tilde{\mathcal Q}$, we have
%\footnotesize
\begin{align*}
&\mathrm{ELBO}(\theta;\mathcal S_{\mathrm{pos}},\mathcal S_{\mathrm{neg}})\\
\;\le\;& \mathrm{ELBO}(\theta;\mathcal Q,\tilde{\mathcal Q})
\;=\; \frac{1}{N}\sum_{i=1}^N \log p_\theta(x^i).
\end{align*}
\normalsize
Moreover, the inequality is strict if either $p_\theta(z\mid x^i)\notin \mathcal S_{\mathrm{pos}}(x^i)$ for a set of $i$ with positive frequency, or $p_\theta(x,z)\notin \mathcal S_{\mathrm{neg}}$.
\end{enumerate}
\end{theorem}
\begin{proof} 
The proof applies the Donsker--Varadhan variational identity \citep{donsker1975asymptotic,wainwright2008graphical,cover2006elements}
$\log \int e^{f}\,d\lambda=\sup_{q\ll\lambda}\{\E_q[f]+H(q)\}$ twice:
\begin{enumerate}[label=(\roman*)]
\item With $f_x(z) =$ $-E(x,z;\theta)$ and base measure $dz$ to express
$\log\!\int e^{-E(x^i,z;\theta)}dz$ as $\sup_{q\in\mathcal Q}\mathcal A_i(q)$,
attained at $q^\star(\cdot)=p_\theta(\cdot\mid x^i)$.
\item With $f(x,z) =-E(x,z;\theta)$ and base measure $dx\,dz$ to express
$\log\!\iint e^{-E(x,z;\theta)}dx\,dz$ as $\sup_{\tilde q\in\tilde{\mathcal Q}}\mathcal B(\tilde q)$,
attained at $\tilde q^\star=p_\theta(x,z)$.
\end{enumerate}
Subtracting the second identity from the first and averaging over \(i\) yields Part~(1).  
Part~(2) follows because taking a supremum over a smaller feasible set cannot increase its value.  
Strictness holds since the optimizers in Part~(1) are (a.e.) unique, which follows from the strict convexity of relative entropy \citep{cover2006elements,rockafellar1970convex}.
\end{proof}

\section{Conditional Generation}
\label{sec:conditional}
In the LV-EBM, reconstructing $x$ from $z$ just means drawing from the conditional distribution
\begin{align}
x \sim p_\theta(x|z) \propto \exp(-E(x,z;\theta)).\label{eq:map}
\end{align}
Consider the following ways:
\begin{enumerate}[label=(\roman*)]
\item Deterministic generation (MAP reconstruction): We can find the most likely $x$ according to $x^* = \arg\max_x\, \log p_\theta(x|z)$. To get $x^*$, we can perform gradient descent given an randomly initiated $x_0$:
\begin{align}
x_{t+1} = x_t - \eta \nabla_x E(x_t, z; \theta).\label{eq:deter_recon}
\end{align}
\item Stochastic generation (Langevin dynamics): Based on \eqref{eq:map}, we can find the set $x$ to be the one that minimizes the energy, $x^* = \arg \min_x E(x,z;\theta)$. In practice, we can run overdamped Langevin dynamics in the $x$-coordinates while holding $z$ fixed: 
\begin{align}
x_{t+1} = x_t - \eta \nabla_x E(x_t, z;\theta) + \sqrt{2\eta}\epsilon_t. \label{eq:stoch_recon}
\end{align}
Alternatively, for better mixing in high-dimensional $x$, we may introduce momentum:
\small
\begin{equation}
\begin{aligned}
v_{k+1} &= (1-\gamma)v_k - \eta \nabla_xE(x_k,z;\theta) + \sqrt{2\gamma \eta} \epsilon_k,\\
x_{k+1} &= x_k + v_{k+1}.
\end{aligned}\label{eq:sghmc}
\end{equation}
\normalsize
% no extra network needed. 
% why may the latter be beter. 
\end{enumerate}
The former approach gives a sharp deterministic MAP estimation of $x$, while the latter preserves multi-modality/uncertainty. 

\section{Experiments}
\label{sec:experiments}
We aim to evaluate the effectiveness of our method for both \emph{posterior inference} and \emph{generative reconstruction}. 
To this end, we consider three controllable geometric families: \textbf{2D concentric rings}, \textbf{3D concentric spheres}, and \textbf{2D harmonic-modulated rings} (Fourier-based radius modulation). 
These datasets allow us to assess whether the models can accurately approximate the conditional posterior.
We compare against \textbf{VAE}, \textbf{Non-Amortized VI (NAVI)} and \textbf{Hard EM} on multiple metrics. All our experiments were conducted on the server with 48-core 3.00GHz Intel(R) Xeon(R) Gold 6248R CPU and 8 NVIDIA A30 GPUs. The code is made available at \url{https://anonymous.4open.science/r/LV_EBM-5730/}. 

\begin{table}[t]
  \centering
  \scriptsize
  \setlength{\tabcolsep}{3pt} 
  \resizebox{\linewidth}{!}{%
  \begin{tabular}{llcccc}
    \toprule
    && Ours & VAE & NAVI & Hard EM \\
    \midrule
    \multirow{4}{*}{LCR-2D}
      & ELBO   & \textbf{2.5048} & 2.3038 & 2.3039 & 2.3025 \\
      & RMSE   & \textbf{0.1612} & 0.7585 & 0.7651 & 0.7655 \\
      & MMD    & \textbf{0.1549} & 0.5712 & 0.5259 & 0.5680 \\
      & $W_2^2$& \textbf{0.2161} & 1.0821 & 1.0787 & 1.0834 \\
    \midrule
    \multirow{4}{*}{LCS-3D}
      & ELBO   & \textbf{3.4984} & 2.8499 & 2.8505 & 2.8274 \\
      & RMSE   & \textbf{0.1303} & 0.6247 & 0.6189 & 0.6225 \\
      & MMD    & \textbf{0.0129} & 0.5376 & 0.5471 & 0.5474 \\
      & $W_2^2$& \textbf{0.2567} & 1.0837 & 1.0748 & 1.0869 \\
    \midrule
    \multirow{4}{*}{HMR-2D}
      & ELBO   & \textbf{5.2234} & 3.1866 & 3.1813 & 3.1850 \\
      & RMSE   & \textbf{0.0375} & 0.8565 & 1.1890 & 1.1881 \\
      & MMD    & \textbf{0.0012} & 0.3485 & 0.4843 & 0.5620 \\
      & $W_2^2$& \textbf{0.1663} & 1.1998 & 1.6820 & 1.6739 \\
    \bottomrule
  \end{tabular}
  }% end resizebox
  \caption{Results across three scenarios, four methods, and four evaluation metrics.}
  \label{tab:main_res}
\end{table}

\subsection*{Datasets}
We consider three datasets and use fixed train/test splits with shared random seeds across methods.

\begin{enumerate}[leftmargin=1.2em,itemsep=0.2em,topsep=0.2em]
\item \textbf{2D concentric rings (LCR-2D).}
Sample \(R_c\in\mathcal C\), \(\phi\sim\mathrm{Unif}[0,2\pi)\), \(\varepsilon_r\sim\mathcal N(0,\sigma_r^2)\), \(\varepsilon_x\sim\mathcal N(0,\sigma_x^2 I_2)\); set
\(r=R_c+\varepsilon_r\), \(x=r\,[\cos\phi,\sin\phi]^\top+\varepsilon_x\).

\item \textbf{3D concentric spheres (LCS-3D).}
Sample \(R_c\in\mathcal C\); draw a uniform direction on \(\mathbb S^2\) by \(u\sim\mathrm{Unif}[-1,1]\), \(\phi\sim\mathrm{Unif}[0,2\pi)\), \(\theta=\arccos u\).
With \(\varepsilon_r\sim\mathcal N(0,\sigma_r^2)\), \(\varepsilon_x\sim\mathcal N(0,\sigma_x^2 I_3)\): \(r=R_c+\varepsilon_r\), \(x=r\,[\sin\theta\cos\phi,\sin\theta\sin\phi,\cos\theta]^\top+\varepsilon_x\).

\item \textbf{2D harmonic-modulated rings (HMR-2D).}
On top of \(R_c\), define \(r(\phi)=R_c+\sum_{k=1}^{K}\!\big[a_k\cos(k\phi)+b_k\sin(k\phi)\big]+\varepsilon_r\) with \(\phi\sim\mathrm{Unif}[0,2\pi)\), \(\varepsilon_r\sim\mathcal N(0,\sigma_r^2)\).
Let \(\varepsilon_x\sim\mathcal N(0,\sigma_x^2 I_2)\), \(x=r(\phi)\,[\cos\phi,\sin\phi]^\top+\varepsilon_x\).
\end{enumerate}

\subsection{Baselines}
\textbf{VAE.}
Encoder \(q_\phi(z\mid x)\) and decoder \(p_\theta(x\mid z)\) are diagonal Gaussians; we optimize the ELBO with the reparameterization trick.

\textbf{Non-Amortized VI (NAVI).}
Per-sample variational parameters \((\mu_i,\log\sigma_i^2)\) are optimized directly via \(K\) inner gradient-ascent steps; then \(\theta\) is updated to maximize the ELBO. No encoder is used, removing the amortization gap.

\textbf{Hard EM.}
E-step: MAP optimization of \(z\) to maximize \(\log p(x\mid z;\theta)+\log p(z)\).
M-step: update \(\theta\) by maximizing the conditional likelihood.

\subsection{Evaluation Metrics}
We report four metrics (arrows indicate the preferred direction):
\begin{enumerate}[leftmargin=1.2em,itemsep=0pt,topsep=2pt]
\item \textbf{ELBO / NLL (\(\uparrow\)).} We report absolute ELBO for our method, VAE, and NAVI, and absolute NLL for Hard EM.
\item \textbf{RMSE (\(\downarrow\)).}Root mean squared error between the reconstruction \(\hat{x}\) and the ground truth \(x\).
\item \textbf{MMD (RBF) (\(\downarrow\)).} Gaussian-kernel MMD; bandwidth via median heuristic or multi-scale.
\item \textbf{\(W_2^2\) (Sinkhorn) (\(\downarrow\)).} Entropy-regularized OT approximation to squared 2-Wasserstein; common \(\varepsilon\) and iteration budget across settings.
\end{enumerate}

\subsection{Main Results}

Table \ref{tab:main_res} summarizes all results across the three settings and four methods. We observe: 1) \textbf{ELBO / NLL:} Our method is better than or competitive with the strongest baseline on all datasets, consistent with direct, energy-driven particle updates in both the joint and conditional spaces that avoid approximation bias from restricted variational families. 2) \textbf{RMSE and MMD (RBF):} Our method shows a clear advantage in generative reconstruction, achieving \emph{lower} RMSE and MMD, indicating more reliable preservation of and alignment with the ground-truth data structure. 3) \textbf{\(W_2^2\) (Sinkhorn):} Our method yields \emph{lower} distances across all three datasets, suggesting better global alignment of the distributions in both summary statistics and distributional shape.

\begin{figure}[!t]
  \centering
  \includegraphics[width=0.95\linewidth]{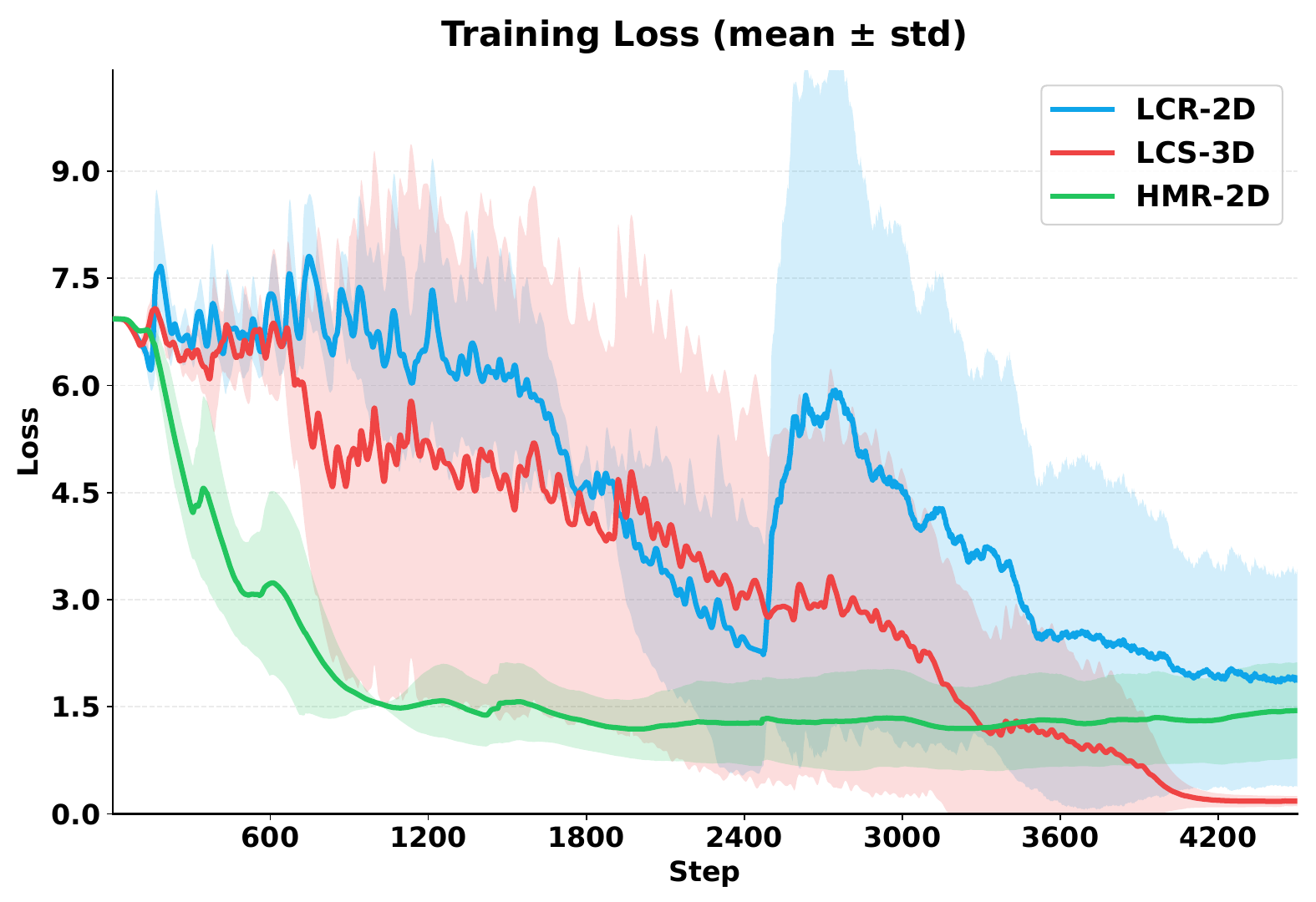}
  \caption{Training loss vs. step (up to 4500) for three scenarios. Solid lines depict the mean over three runs; shaded regions show ±1 standard deviation.}
  \label{fig:loss}
\end{figure}

To illustrate the training dynamics and convergence of our method, we plot the training loss for each setting together with the mean \(\pm\) standard deviation over three independent runs. As shown in Figure \ref{fig:loss}, the objective decreases before plateauing, and variance decreases steadily, indicating stable convergence. This behavior is consistent with theoretical results on entropy decay and Wasserstein contraction for particle Langevin updates under standard assumptions. 

\begin{figure*}[t]
  \centering
  \includegraphics[width=0.57\textwidth]{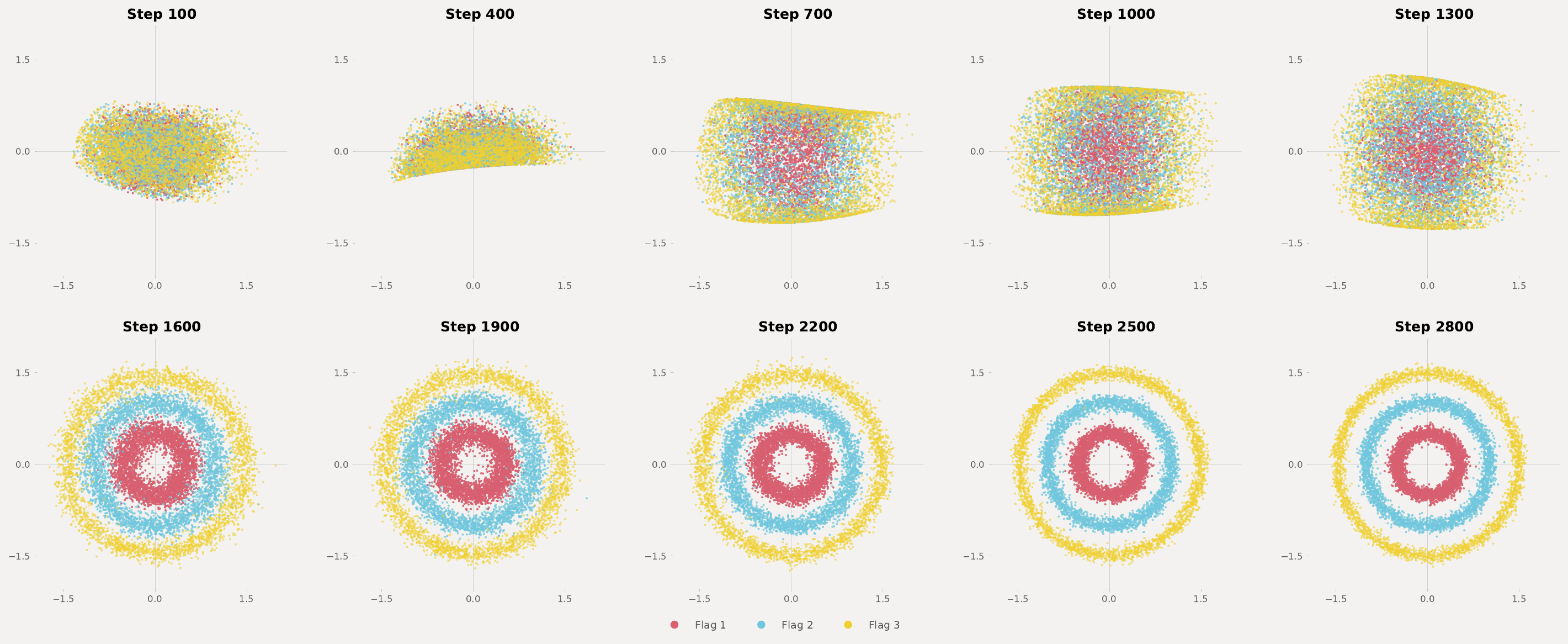}
  % \hfill
  \includegraphics[width=0.37\textwidth]{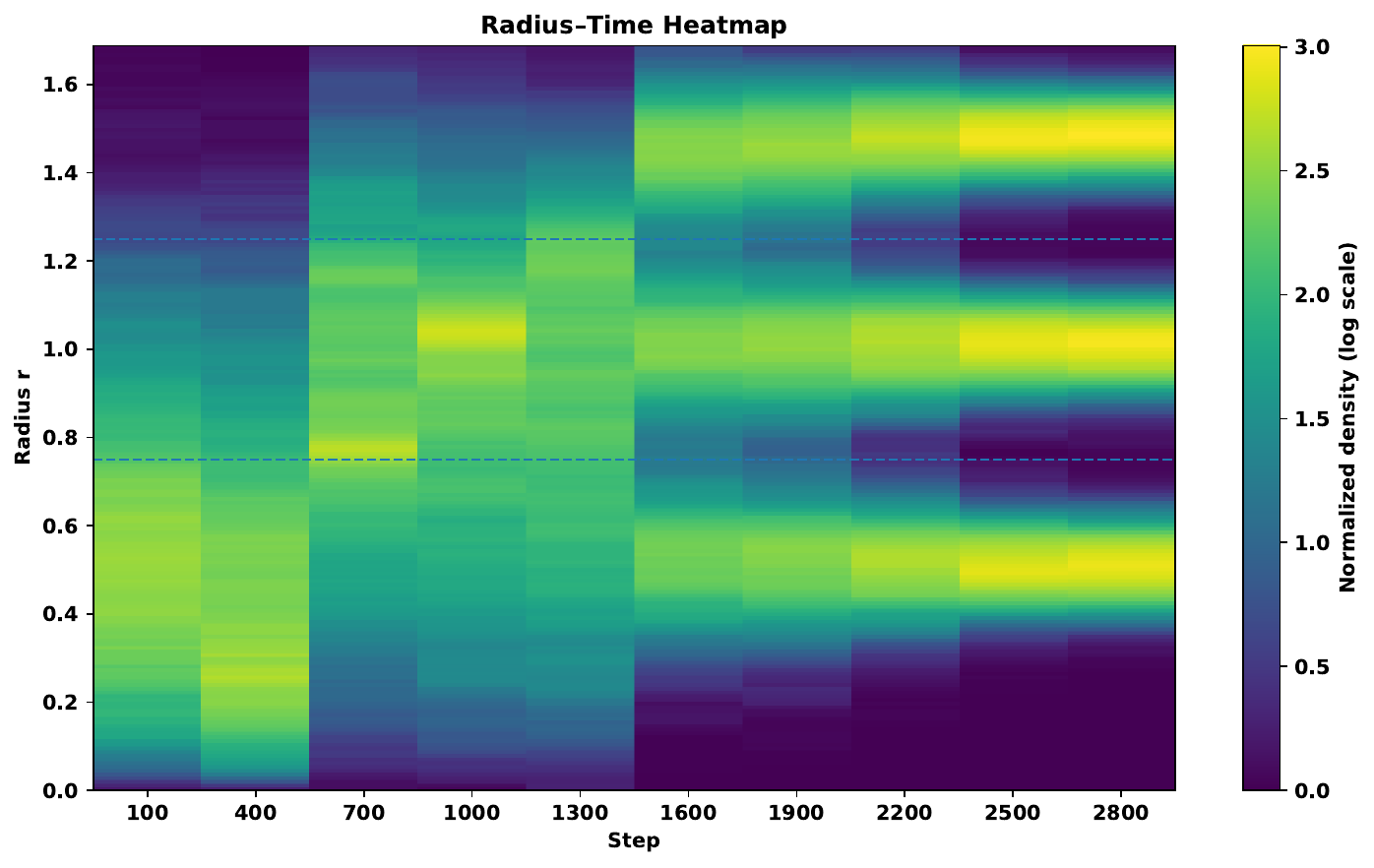}
  \caption{\textbf{Left}: Ten snapshots of LCR-2D reconstructions across training steps; Colors denote ground-truth radial classes. 
  \textbf{Right}: Radius–step heatmap computed from the same reconstructions. The color intensity encodes the per-step normalized radial density.}
  \label{fig:diffusion_process}
\end{figure*}

\textbf{Reconstruction Method.}
For each observation \(x\), we run a short-run Langevin posterior probe under the energy \(E(x,z)\) to produce \(M\) candidates and pick the MAP code \(\hat z=\arg\min_{z} E(x,z)\).
With the energy network frozen, we train a decoder \(g_\theta\) by minimizing \(\lVert g_\theta(\hat z)-x\rVert_2^2\) over the training set. At inference, we repeat the probe to obtain \(\hat z\) and output the reconstruction \(x_{\text{rec}}=g_\theta(\hat z)\). We visualize the training snapshots for LCR-2D as shown in Figure \ref{fig:diffusion_process}. Points are colored by three radius-based classes. Test particles evolve from a scattered layout to tight alignment with the target rings; inter-class boundaries become clear and within-class density becomes more uniform, showing good mixing and mode preservation.

Figure \ref{fig:recon_compare} demonstrates the alignment between reconstructed and ground-truth distributions under three settings. The reconstructions closely track the ground truth in ring or shell thickness.

\begin{figure}[!htbp]
  \centering
  \includegraphics[width=0.92\linewidth]{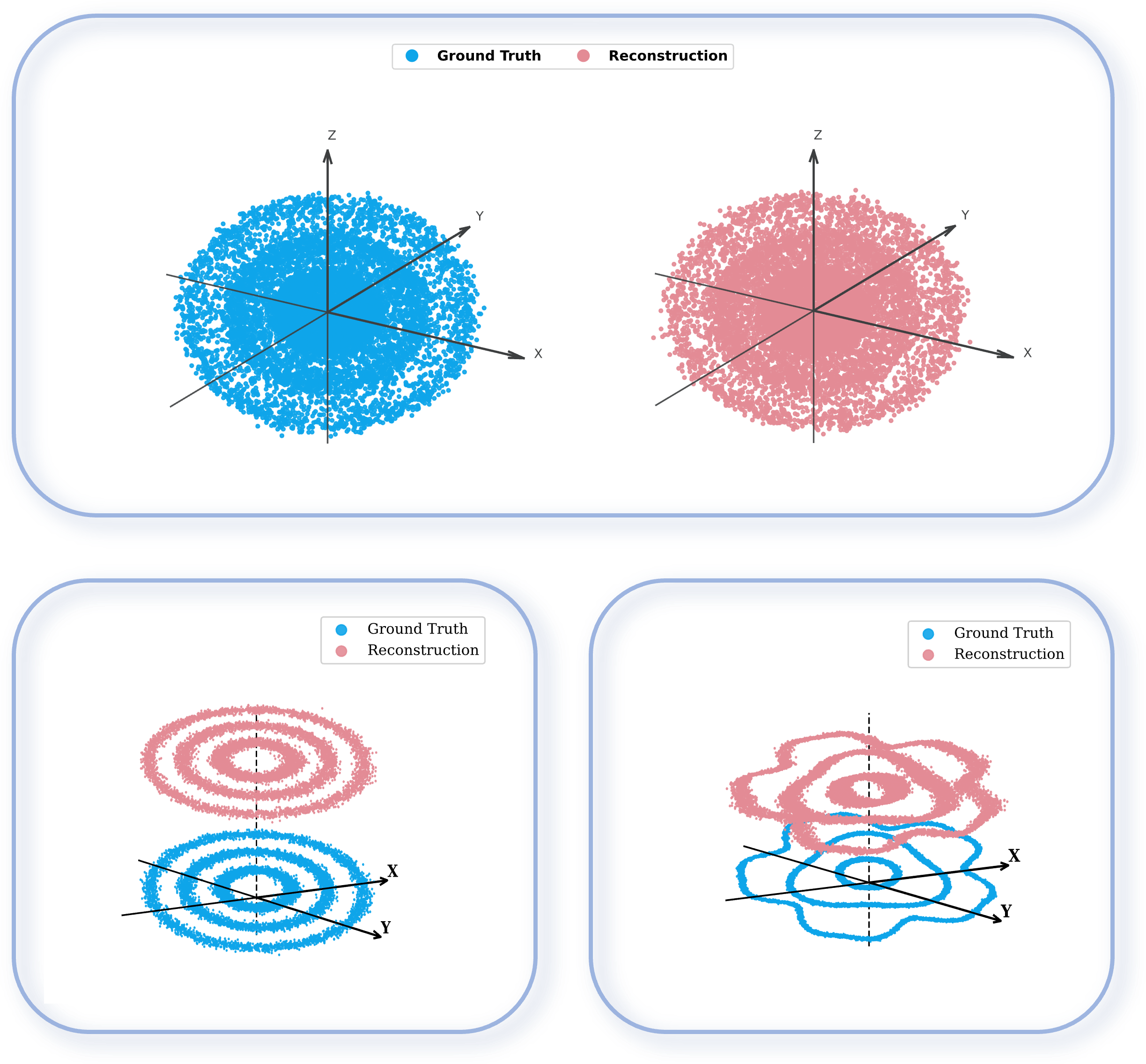}
  \caption{Reconstruction results: LCS-3D (top), LCR-2D (bottom-left), and HMR-2D (bottom-right).}
  \label{fig:recon_compare}
\end{figure}

\subsection{Discussion and Summary}
Across the three tasks and four metrics, our method achieves stable gains in both \emph{probabilistic consistency} (ELBO/NLL) and \emph{geometric/perceptual consistency} (RMSE, MMD (RBF), \(W_2^2\) (Sinkhorn)). 
We attribute this to two key factors: (i) coupled particle--Langevin updates for the joint \((x,z)\) and the conditional \(z\!\mid\!x\), which fit the target Gibbs distribution more directly and reduce bias from restricted variational families or decoder assumptions; and (ii) MAP/Langevin reconstruction along the energy landscape, which preserves multi-modal geometric structure and frequency-related patterns. 
These observations are consistent with the theoretical motivation and guarantees discussed earlier (particle implementation, Fokker--Planck evolution, and KL/\(W_2\) contraction), and they suggest that the approach scales to more complex data.

\section{Conclusion}

We presented a principled framework for latent-variable energy-based models (LV-EBMs) via coupled Wasserstein gradient flows, casting MLE as a saddle problem over latent and joint distributions. This yields a compact, particle-based learning algorithm that alternates conditional latent updates with joint negative sampling, requiring no auxiliary posterior, discriminator, or decoder, and whose induced Langevin dynamics enjoy existence, stability, and exponential convergence guarantees under standard smoothness and dissipativity—measured in KL and Wasserstein metrics. Empirically, across controlled physical systems, the method is consistently competitive or superior to VAE, non-amortized VI, and hard-EM baselines, reflecting stronger preservation of multimodal geometric structure and tighter likelihood objectives. Looking ahead, we see opportunities to scale to high-dimensional real data, enrich energy parameterizations, and interface with diffusion/flow-based models, further unifying theory and practice for LV-EBMs.

%We use the following particle dynamics:
%\begin{enumerate}
%\item $\epsilon$
%\end{enumerate}

% prove why stochastic particles realize the WGF
% pos sample neg samples, looks like GAN but not exacly.

%\newpage

\bibliographystyle{plain} % or abbrv, unsrt, etc. (depends on AISTATS style)
\bibliography{refs}

\onecolumn
\aistatstitle{Particle Dynamics for Latent-Variable Energy-Based Models}

\section{Theorem 1 with a Weaker Assumption}

In this section we state a weaker version of Theorem~1 that applies to the conditional Fokker--Planck flow in $z$ with $x$ fixed. 
Relative to the main text, we replace the stronger tail/functional-inequality assumptions (e.g., LSI) by simple coercivity and dissipativity of $E(x,\cdot)$; this is sufficient to ensure normalizability of $p^*(z\mid x)\propto e^{-E(x,z)}$ and well-posedness of the dynamics. 
The price is that we obtain weak convergence without a rate (no exponential contraction). 
This is exactly what is needed to justify the conditional Langevin inner loop used by our algorithm; whenever stronger curvature/tail controls hold, the main-text result with rates applies.

\begin{theorem}[Weak convergence without LSI]
	Fix $x$ and consider the conditional dynamics in $z\in\mathbb{R}^\ell$. Let $E(x,z)$ satisfy, for a.e.\ $x$:
	\begin{enumerate}
		\item \textbf{Coercive growth:} $E(x,z)\to\infty$ as $\lvert z\rvert\to\infty$.
		\item \textbf{Smoothness:} $E(x,\cdot)\in C^{2}(\mathbb{R}^{\ell})$ and $\nabla_{z}E(x,\cdot)$ is locally Lipschitz.
		\item \textbf{Dissipativity:} $\langle z,\nabla_{z}E(x,z)\rangle \ge m\lvert z\rvert^{2}-b$ for some $m>0$, $b\ge 0$.
		\item \textbf{Well-posedness (FP/SDE):} The conditional FP equation below (equivalently, the Langevin SDE) is non-explosive and has a unique weak solution for any $q_0^x\in\mathcal P_2(\mathbb{R}^{\ell})$.
	\end{enumerate}
	Let $q_t^x$ solve the conditional Fokker--Planck equation in $z$:
	\[
	\partial_t q_t^x(z)
	= \nabla_{z}\!\cdot\!\big(q_t^x(z)\,\nabla_{z} E(x,z)\big) + \Delta_{z} q_t^x(z),
	\qquad q_{t=0}^x = q_0^x \in \mathcal P_2(\mathbb{R}^{\ell}).
	\]
	Let $p^*(z\mid x)\propto e^{-E(x,z)}$ be the conditional Gibbs measure (well-defined by coercivity).
	Then $q_t^x \rightharpoonup p^*(\cdot\mid x)$ weakly as $t\to\infty$.
\end{theorem}

\begin{proof}
	\textbf{1. Entropy functional.}
	Define the relative entropy $\mathcal H(t):=D_{\mathrm{KL}}(q_t^x\|p^*(\cdot\mid x))$.
	
	\textbf{2. Entropy dissipation.}
	Using $\partial_t q_t^x=\nabla_z\!\cdot(q_t^x\nabla_z E)+\Delta_z q_t^x$ and integrating by parts,
	\begin{align}
		\frac{d\mathcal H}{dt}
		&= \int \big(\partial_t q_t^x\big)\!\left(\log\frac{q_t^x}{p^*}+1\right)\,dz \notag\\
		&= \int \big[\nabla_z\!\cdot(q_t^x\nabla_z E)+\Delta_z q_t^x\big]\,
		\log\frac{q_t^x}{p^*}\,dz \notag\\
		&= -\!\int q_t^x\,\nabla_z E\cdot\nabla_z \log\frac{q_t^x}{p^*}\,dz
		-\!\int \nabla_z q_t^x\cdot\nabla_z \log\frac{q_t^x}{p^*}\,dz \notag\\
		&= -\!\int \big(q_t^x\nabla_z E+\nabla_z q_t^x\big)\cdot\nabla_z \log\frac{q_t^x}{p^*}\,dz \notag\\
		&= -\!\int q_t^x\big(\nabla_z E+\nabla_z \log q_t^x\big)\cdot\nabla_z \log\frac{q_t^x}{p^*}\,dz \notag\\
		&= -\!\int q_t^x\big(-\nabla_z \log p^*+\nabla_z \log q_t^x\big)\cdot\nabla_z \log\frac{q_t^x}{p^*}\,dz \notag\\
		&= -\int q_t^x \left\lvert \nabla_z \log\frac{q_t^x}{p^*}\right\rvert^2 dz
		= - I(q_t^x\mid p^*), \label{eq:entropy-diss}
	\end{align}
	where $I(q\mid p):=\int q\,\lvert \nabla \log(q/p)\rvert^2 dz$ is the relative Fisher information. Hence $\mathcal H$ is nonincreasing.
	
	\textbf{3. Integrability of Fisher information.}
	From \eqref{eq:entropy-diss},
	\[
	\mathcal H(t) = \mathcal H(0)-\int_{0}^{t} I(q_s^x\mid p^*)\,ds,
	\quad\Rightarrow\quad
	\int_{0}^{\infty} I(q_s^x\mid p^*)\,ds \le \mathcal H(0)<\infty.
	\]
	
	\textbf{4. Second moment and tightness.}
	Let $M_2(t):=\int \lvert z\rvert^2 q_t^x(z)\,dz$. Then
	\begin{align*}
		\frac{dM_2}{dt}
		&= \int \lvert z\rvert^2\,\partial_t q_t^x\,dz
		= \int \lvert z\rvert^2\big[\nabla_z\!\cdot(q_t^x\nabla_z E)+\Delta_z q_t^x\big]\,dz \\
		&= -\int \nabla_z(\lvert z\rvert^2)\cdot (q_t^x\nabla_z E)\,dz \;+\; \int q_t^x\,\Delta_z(\lvert z\rvert^2)\,dz \\
		&= -2\int q_t^x \langle z,\nabla_z E\rangle dz \;+\; 2\ell,
	\end{align*}
	since $\nabla_z(\lvert z\rvert^2)=2z$ and $\Delta_z(\lvert z\rvert^2)=2\ell$.
	By dissipativity,
	\[
	\frac{dM_2}{dt}\;\le\; -2m\,M_2(t) + 2b + 2\ell,
	\]
	so by Grönwall,
	\[
	M_2(t)\;\le\; e^{-2mt}M_2(0) + \frac{b+\ell}{m}\big(1-e^{-2mt}\big)
	\;\le\; M_2(0)+\frac{b+\ell}{m},
	\]
	which yields $\sup_{t\ge 0}M_2(t)<\infty$ and tightness of $\{q_t^x\}_{t\ge 0}$.
	
	\textbf{5. Subsequence limits and identification.}
	Tightness gives $t_n\to\infty$ with $q_{t_n}^x \rightharpoonup q_\infty$.
	Since $\int_0^\infty I(q_s^x\mid p^*)\,ds<\infty$, there exists a (relabeled) subsequence with $I(q_{t_n}^x\mid p^*)\to 0$.
	Lower semicontinuity of entropy and Fisher information under weak convergence implies
	\[
	D_{\mathrm{KL}}(q_\infty\|p^*) \le \liminf_{n} D_{\mathrm{KL}}(q_{t_n}^x\|p^*),\qquad
	I(q_\infty\mid p^*) \le \liminf_{n} I(q_{t_n}^x\mid p^*)=0,
	\]
	hence $I(q_\infty\mid p^*)=0$, which forces $q_\infty=p^*$.
	
	\textbf{6. Conclusion.}
	Every convergent subsequence has the same limit $p^*$, so $q_t^x\rightharpoonup p^*$.
\end{proof}

\paragraph{Remark.}
Compared with the stronger LSI-based statement (which yields exponential rates), the assumptions above only guarantee weak convergence without a rate, but hold under much milder tail and dissipativity conditions.

\section{Experiments on Real-World Data}

\subsection*{UCI POWER}
\paragraph{Origin \& semantics.}
The POWER dataset records household electrical power consumption over 47 months at one-minute resolution. Although it is a time series, prior density-estimation work treats each record as an i.i.d.\ sample from the marginal distribution. The commonly used preprocessed version converts time-of-day to minutes, adds small uniform jitter to avoid duplicate values, discards the date field, and removes ``global reactive power'' due to many exact zeros. The resulting dimensionality is $D=6$.

\subsection*{UCI MINIBOONE}
\paragraph{Origin \& semantics.}
From the MiniBooNE neutrino oscillation experiment, originally a classification task. The preprocessed density-estimation version uses only the positive class (electron neutrinos), removes 11 obvious outliers with $-1000$ in every column, and drops seven features with extreme value counts. The resulting dimensionality is $D=43$.

\begin{table}[t]
	\centering
	\caption{Reference statistics for the MAF-preprocessed POWER and MINIBOONE datasets (as distributed), plus our experimental subsampling protocol.}
	\label{tab:uci_ref_stats}
	\begin{tabular}{lrrrrl}
		\toprule
		\textbf{Dataset} & \textbf{Dim.} & \textbf{Train} & \textbf{Valid} & \textbf{Test} & \textbf{Provenance} \\
		\midrule
		POWER      & 6  & 1,659,917 & 184,435 & 204,928 & MAF bundle\footnotemark[4] \\
		MINIBOONE  & 43 & 29,556    & 3,284   & 3,648   & MAF bundle\footnotemark[4] \\
		\midrule
		\multicolumn{6}{l}{\textit{Our protocol (for experiments): train 20k, test 2k (fixed seed).}} \\
		\bottomrule
	\end{tabular}
\end{table}

\paragraph{Provenance of the preprocessed data.}
We use the exact public bundle released for the Masked Autoregressive Flow (MAF) experiments (Zenodo record \#1161203), which standardizes the community setup across POWER, GAS, and MINIBOONE.\footnote{Zenodo record \#1161203: \url{https://zenodo.org/record/1161203}}

\subsection*{How we use the datasets (training/evaluation protocol)}

For comparability across methods and compute budgets, we randomly sample (with a fixed seed) $20{,}000$ examples from the train pool for training and $2{,}000$ examples from the test split for evaluation.

\subsection{Results on UCI Data}

We adopt the same four evaluation metrics as those used in the main text for experiments. The experimental results are presented in Table~\ref{tab:uci_res}. The results show that our method and the baseline models are trained on the same training set and evaluated on the same test set. Across all four evaluation metrics, our method consistently outperforms the baselines.

\begin{table}[H]
	\centering
	\begin{tabular}{llcccc}
		\toprule
		&& Ours & VAE & NAVI & Hard EM \\
		\midrule
		\multirow{4}{*}{UCI-Power}
		& ELBO   & \textbf{8.2616} & 7.5651 & 7.5971 & 3.8666 \\
		& RMSE   & \textbf{0.3368} & 0.3705 & 0.3630 & 0.6671 \\
		& MMD    & \textbf{0.0957} & 0.1159 & 0.1133 & 0.2231 \\
		& $W_2^2$& \textbf{0.2331} & 0.2855 & 0.2854 & 0.2570 \\
		\midrule
		\multirow{4}{*}{UCI-Miniboone}
		& ELBO   & \textbf{35.7191} & 32.3384 & 31.5001 & 31.9774 \\
		& RMSE   & \textbf{0.2953} & 0.3410 & 0.3273 & 0.3224 \\
		& MMD    & \textbf{0.0542} & 0.0751 & 0.0766 &  0.1032\\
		& $W_2^2$& \textbf{0.0012} & 0.0014 & 0.0024 & 0.0021 \\
		\bottomrule
	\end{tabular}
	\caption{Results on UCI POWER and MINIBOONE using four evaluation metrics.}
	\label{tab:uci_res}
\end{table}

\end{document}